\documentclass[11pt,a4paper]{article}

\usepackage{graphicx}
\usepackage{hyperref} 
\usepackage{amsmath,amsfonts}
\usepackage{amsthm}

\usepackage{tabularx}
\usepackage{enumerate}
\usepackage{algorithm}
\usepackage{algpseudocode}
\usepackage{xspace}

\usepackage{authblk}
\usepackage{multirow}

\usepackage{longtable}
\usepackage{graphicx}
\usepackage{subfig}
\usepackage{mathrsfs}
\usepackage{mathabx}
\usepackage[normalem]{ulem} 
\usepackage{setspace}
\usepackage{float} 

\usepackage[colorinlistoftodos]{todonotes}
\usepackage[utf8]{inputenc}

\def\R{\ensuremath{\mathbb{R}}}
\def\C{\ensuremath{\mathbb{C}}}

\def\L2{L^{2}(\mathbb{R}^{n})}

\def\<{\ensuremath{\langle}}
\def\>{\ensuremath{\rangle}}

\newtheorem{proposition}{Proposition}[section]
\newtheorem{lemma}{Lemma}[section]

\newtheorem{definition}{Definition}[section]

\def\I{\mathbb {I}}
\def\R{\mathbb{R}}

\def\L{\mathcal{L}}

\graphicspath{{Figures/}} % Directory in which figures are stored

\allowdisplaybreaks

\providecommand{\keywords}[1]
{
  \small	
  \textbf{\textit{Keywords---}} #1
}

\definecolor{lime}{HTML}{A6CE39}
\DeclareRobustCommand{\orcidicon}{%
	\begin{tikzpicture}
	\draw[lime, fill=lime] (0,0) 
	circle [radius=0.16] 
	node[white] {{\fontfamily{qag}\selectfont \tiny ID}};
	\draw[white, fill=white] (-0.0625,0.095) 
	circle [radius=0.007];
	\end{tikzpicture}
	\hspace{-2mm}
}

\foreach \x in {A, ..., Z}{%
	\expandafter\xdef\csname orcid\x\endcsname{\noexpand\href{https://orcid.org/\csname orcidauthor\x\endcsname}{\noexpand\orcidicon}}
}
\title{Divergence Phase Index: A Riesz-Transform Framework for Multidimensional Phase Difference Analysis}

\author[1]{Magaly Catanzariti}
\author[1,*]{Hugo Aimar\orcidB{}}
\author[1,2,3,*] {Diego M. Mateos\orcidA{}}

\affil[1]{\normalsize Instituto de Matem\'{a}tica Aplicada del Litoral (IMAL), UNL, CONICET. Santa Fe, Argentina.}
\affil[2]{\normalsize Facultad de Ciencia y Tecnolog\'{\i}a, Universidad Aut\'{o}noma de Entre R\'{\i}os (UADER). Oro Verde, Entre R\'{\i}os, Argentina.}
\affil[3]{\normalsize Achucarro Basque Center for Neuroscience. Leioa, Vizcaya, Spain.}

\affil[*]{Corresponding authors: Diego M. Mateos (\texttt{mateosdiego@gmail.com}), Hugo Aimar (\texttt{haimar@santafe-conicet.gov.ar})}
\begin{document}
\date{}
\maketitle

%================= ABSTRACT
\begin{abstract}
We introduce the Divergence Phase Index (DPI), a novel framework for quantifying phase differences in one and multidimensional signals, grounded in harmonic analysis via the Riesz transform. Based on classical Hilbert Transform phase measures, the DPI extends these principles to higher dimensions, offering a geometry-aware metric that is invariant to intensity scaling and sensitive to structural changes. We applied this method on both synthetic and real-world datasets, including intracranial EEG (iEEG) recordings during epileptic seizures, high-resolution microscopy images, and paintings. In the 1D case, the DPI robustly detects hypersynchronization associated with generalized epilepsy, while in 2D, it reveals subtle, imperceptible changes in images and artworks. Additionally, it can detect  rotational variations in highly isotropic microscopy images. The DPI’s robustness to amplitude variations and its adaptability across domains enable its use in diverse applications from nonlinear dynamics, complex systems analysis, to multidimensional signal processing.
\end{abstract}

%%===========================================================
\keywords{ Divergence Phase Index, Riesz transform, Multidimensional phase analysis, Phase synchronization, Image similarity metrics, Neurophysiological Signal Analysis. }

%\bigskip

%\bigskip
%==================================================
%                   Introduction
%==================================================
\section{Introduction}
The study of the synchronization between signals is one of the most important goals in the field of one-dimensional signal analysis \cite{Pikovsky_2001}. To address this, various measures have been developed to quantify the similarity or dissimilarity between signals. Among the most widely used approaches are those based on analysing phase differences extracted via the Hilbert transform. A pioneering contribution in this area came from Mormann et al. \cite{Mormann_2000}, introducing the notion of phase coherence. However, conventional definitions of phase difference are intrinsically one-dimensional, making their extension to multidimensional signals—such as images or spatial maps—nontrivial.

Building based on these ideas, this work develops a mathematically rigorous framework that extends the notions of instantaneous phase difference  from one-dimensional signals to higher-dimensional data. Starting from classical harmonic analysis, we formalized the Hilbert transform using the Fourier representation of Schwartz distributions. This naturally leads to a definition of instantaneous phase based on analytic signals. We then use the theory of singular integrals, particularly Riesz transforms, to develop a generalised version of the Hilbert transform that is suitable for multidimensional signals. This enables us to define a \textit{Divergence Phase Index} ($DPI$), which can be used to evaluate phase differences between scalar or tensor-valued functions defined in Euclidean spaces.

We applied our algorithm to concrete examples that demonstrate the application of $DPI$ to real and synthetic data. In one-dimensional cases, we analysed intracranial EEG signals from an epileptic patient. In higher dimensions, we show how $DPI$ distinguishes between images (generated and real) based on structural differences, even when intensity variations are present.
Finally, we present the properties for detecting rotation in highly isotropic biological samples.
These results emphasize the potential of $DPI$ as a robust, geometry-aware tool for the study of synchronization and structural differences in a wide range of scientific and engineering contexts.

Section~\ref{sec:theoretical} formalizes distributional Fourier/Hilbert theory, from the point of view of analytic functions. In Section~\ref{HTsignal}, we introduce the concept of instantaneous phase of a signal on the basis of the Hilbert Transform and the Divergence Phase Index ($DPI$) in one dimension. Section~\ref{HigherExtension} constructs the Riesz-based phase vector $\vec{\Phi}_f$  and gives the definition of DPI in dimension larger than one. In Section~\ref{sec:results} we apply $DPI$ to neurophysiological and imaging data. Finally, in Section~\ref{sec:discussion} we discuss the implications of the concept for multidimensional signal processing.

\section{Theoretical Framework}\label{sec:theoretical}

In this section, we aim to introduce the Hilbert transform from the complex analytic point of view, with no appeal to the somehow less intuitive definition of the Hilbert transform via the principal value of a convolution singular integral with a given signal. In doing so, we aim to clear the way to the main point of this note, i.e., the notion of phase for higher dimensional signals.

Usually the complex variable $z$ is taken to be $x+iy$, with $x$ and $y$ in $\R$. Since the real part of $z$ in our one dimensional analysis is meant to be time, we shall instead use $z=t+iy$ with $t$ and $y$ real numbers to denote the complex numbers.

Lets us start by a very simple situation that leads to a quite intrinsic relation between the two basic trigonometric functions $c(t)=\cos{t}$ and $s(t)=\sin{t}$ in $\R$.
The function $F(z)=e^{iz}=e^{-y}e^{it}$ is analytic in the whole complex plane. Its restriction to the real axis (time axis) is given by $F(t)=e^{it}=\cos{t}+i\sin{t}=c(t)+is(t)$. In other words, $s(t)$ can be seen as the imaginary part of the restriction to the real axis of the analytic function $F$ in $\C$ that has $c(t)$ as the real part of this restriction. This procedure can be seen as an operator applying the real part of the restriction of $F$ to the real axes into the imaginary part of this restriction. Let us precise the above. Given $F(z)=u(z)+iv(z)$, with $u$ and $v$ real, be an analytic function in $\C$. Set $\varphi(t)=u(t,0)$, the restriction of the real part of $F$ to the real axis. Define $(H\varphi)(t)=v(t,0)$, the restriction to the real axis of the imaginary part $v$ of $F$. With this notation and $F(z)=e^{iz}$ as above, we have that $s(t)=(Hc)(t)$.
Hence $\frac{(Hc)(t)}{c(t)}=\frac{s(t)}{c(t)}=\tan{t}$, so that $t=\arctan{\frac{(Hc)(t)}{c(t)}}$ for every $t\in\R$.

The operator $H$ heuristically defined above is the Hilbert transform. The critical point of the above approach to $H$ is the fact that not ``every signal'' $\varphi(t)$ is the restriction to the real axis $\I m\,z=0$, of an analytic function on the whole plane $\C$, or even in some open neighbourhood of the upper half plane $\{\I m\,z>0\}=\{ z=t+iy:y>0 \}$.
Nevertheless, using the Poisson kernel and the Cauchy-Riemann equations, it is still possible to extend the naive approach introduced above to the case of non necessarily analytic signals $\varphi(t)$. The basic algorithm to obtain $H\varphi$ is the following. First, find a harmonic function $u$ in $\{\I m\,z>0 \}$ with boundary value $u(t,0)=\varphi(t)$, then solve the Cauchy-Riemann system to find a conjugate harmonic function $v$ of $u$, and finally restrict $v$ to the real axis. The value $v(t,0)$ is the Hilbert transform of $\varphi(t)$. Briefly, $H\varphi(t)=v(t,0)$ where $v$ is a conjugate harmonic of $u(t,y)$ with $u(t,0)=\varphi(t)$. We aim here to give an explicit formula for this operator action based on the Fourier Transform.

\subsection{Analytic functions and the Fourier version of the Cauchy-Riemann system}

Given an integrable one-dimensional signal $f(t)$, usually denoted by $f\in L^1(\R)$, the space of Lebesgue integrable functions, its Fourier Transform is defined by
\begin{equation}
    \widehat{f}(\xi)=\int_{t\in\R}f(t)e^{-2\pi i \xi t}dt.
\end{equation}
A well-known result, the Riemann-Lebesgue Lemma, shows that $\widehat{f}$ is continuous and tends to zero when $|\xi|\to 0$. When $f$ is smooth and vanishes at infinity, after a simple integration by parts, we see that
\begin{equation} \label{eq:RLLemma}
    \widehat{\frac{df}{dt}}(\xi)=2\pi i \xi \widehat{f}(\xi).
\end{equation}

Recall that given a $\mathcal{C}^2$ function $F(z)=u(z)+iv(z)$ defined on $\{\I m\,z>0\}$, with $u(z)$ and $v(z)$ real valued,the function $F$ is analytic in $\{\I m \,z>0\}$ if and only if
\begin{equation} \label{eq:CR}
\text{(CR)} \quad 
\begin{cases}
\dfrac{\partial u}{\partial t}(t,y) = \dfrac{\partial v}{\partial y}(t,y), \\[1.2ex]
\dfrac{\partial v}{\partial t}(t,y) = -\,\dfrac{\partial u}{\partial y}(t,y),
\end{cases}
\end{equation}
for $t\in\R$ and $y>0$.

Now for $y>0$ fixed, set $\widehat{u}(\xi,y)$ to denote the Fourier Transform of $u(t,y)$ as a function of $t$. In other words, $\widehat{u}(\xi,y)=\int_{t\in\R}u(t,y)e^{-2\pi i \xi t}dt$.

Similarly for $\widehat{v}(\xi,y)$. From \eqref{eq:RLLemma}, the Cauchy-Riemann system can be written as
\begin{equation}\label{eq:CRF}
\text{(CRF)} \quad 
\begin{cases}
2\pi i\xi \widehat{u} (\xi,y) = \dfrac{\partial}{\partial y}\widehat{v}(\xi,y), \\[1.2ex]
2\pi i\xi \widehat{v} (\xi,y) =- \dfrac{\partial}{\partial y}\widehat{u}(\xi,y),
\end{cases}
\end{equation}
for $\xi\in\R$ and $y>0$.

\subsection{The Hilbert Transform of a given signal}
\label{subsec:HTsignal}
With the above version of the Cauchy-Riemann equation through the Fourier Transform, we are in position to describe the Hilbert operator in an operational way, avoiding the computation of principal values of singular integrals.

\begin{lemma} \label{lemma2.1}
    Let $\varphi(t)$ be a real signal defined on $\R$. Set $u(t,y)$ for $t\in\R$ and $y>0$, to be such that
    \begin{equation*}
        \widehat{u}(\xi,y)=e^{-2\pi y |\xi|}\widehat{\varphi}(\xi).
    \end{equation*}
    Then
    \begin{enumerate}[(1)]
        \item  $u(t,y)$ is harmonic in $\{\I m\,z>0\}$;
        \item $u(t,0)=\varphi(t)$, $t\in\R$;
        \item the function $v(t,y)$ such that $\widehat{v}(\xi,y)=-i\sigma(\xi)\widehat{u}(\xi,y)$ with $\sigma(\xi)=1$ if $\xi>0$, $\sigma(\xi)=-1$ if $\xi<0$ is a harmonic conjugate of $u(t,y)$;
        \item $\widehat{v}(\xi,0)=-i\sigma(\xi)\widehat{\varphi}(\xi)$.
    \end{enumerate}
\end{lemma}

\begin{proof}
    \begin{enumerate}[(1)]
        \item Let us compute the Laplacian of $u$ by taking Fourier Transform in the variable $t$ using twice \eqref{eq:RLLemma} in Section~2.1  and the definition of $u$,
        \begin{align*}
            \left( \frac{\partial^2 u}{\partial t^2} + \frac{\partial^2u}{dy^2}\right)^{\wedge}(\xi,y)=& \widehat{\frac{\partial^2u}{dt^2}}(\xi,y)+\frac{\partial^2\widehat{u}}{dy^2}(\xi,y) \\
            =& -4\pi^2\xi^2\widehat{u}(\xi,t)+\frac{\partial^2}{dy^2}e^{-2\pi y|\xi|}\widehat{\varphi}(\xi)\\
            =& -4\pi^2 \xi^2\widehat{u}(\xi,y)+4\pi^2\xi^2\left(e^{-2\pi y|\xi|}\widehat{\varphi}(\xi)\right)\\
            =& -4\pi^2 \xi^2 \widehat{u}(\xi,y)+4\pi^2 \xi^2\widehat{u}(\xi,t)\\
            =& 0.
        \end{align*}
        So that $\frac{\partial^2u}{dt^2}+\frac{\partial^2u}{dy^2}=0$ on $\{\I m\,z>0\}$
        \item Taking $y=0$ in the definition of $\widehat{u}(\xi,y)$, we have that $\widehat{u}(\xi,0)=\widehat{\varphi}(\xi)$. Hence $u(t,0)=\varphi(t)$, as desired.
        \item We only need to check that the pair of functions $\widehat{u}(\xi,y)=e^{-2\pi y|\xi|}\widehat{\varphi}(\xi)$ and $\widehat{v}(\xi,y)=-i\sigma(\xi)\widehat{u}(\xi,y)=-i \sigma(\xi)e^{-2\pi y |\xi|}\widehat{\varphi}(\xi)$ satisfy the system $(CRF)$ \eqref{eq:CRF}. Let us check the first one
        \begin{align*}
            2\pi i \xi \widehat{u}(\xi,y)=&-i\sigma(\xi)(-2\pi|\xi|)\widehat{u}(\xi,y)\\
            =& -i\sigma(\xi)(-2\pi|\xi|)e^{-2\pi|\xi|y}\widehat{\varphi}(\xi)\\
            =& -i\sigma(\xi)\frac{\partial}{\partial y}\left( e^{-2\pi |\xi|y}\widehat{\varphi}(\xi)\right)\\
            =& \frac{\partial}{\partial y}\widehat{v}(\xi,y).\\
        \end{align*}
        For the second, we have
        \begin{align*}
            2\pi i \xi \widehat{v}(\xi,y)=& (2\pi i \xi)\left(-i\sigma(\xi)\widehat{u}(\xi,y)\right)\\
            =& 2\pi|\xi|\widehat{u}(\xi,y)\\
            =& -\frac{\partial}{\partial y}\widehat{u}(\xi,y).
        \end{align*}
        \item Follows from $(3)$ taking $y=0$.
    \end{enumerate}
\end{proof}

\begin{definition}
    The function $v(t,0)$ in $(4)$, given by $\psi(t)=(-i\sigma\widehat{\varphi})^\vee(t)$ is the Hilbert Transform of $\varphi$ and is denoted by $H\varphi(t)$, where $\vee$ denotes the inverse Fourier Transform.
\end{definition}

\section{The instantaneous phase of a signal}\label{HTsignal}
As noticed in the previous section, the operator $H$ satisfies $H(C) = S$ where $C(t)=\cos{t}$ and $S(t)=\sin{t}$, we see that

\begin{equation*}
 t=\tan^{-1}\frac{\sin t}{\cos t}=\tan^{-1}\frac{Hf(t)}{f(t)}.
\end{equation*}
So that we say that the instantaneous phase of $C(t)$ is $\phi_f(t) = \tan^{-1}\frac{Hf(t)}{f(t)}=t$. On the other hand, with $S(t)=\sin t$, since from $H\widehat{\varphi}(t)=-i\sigma\widehat{\varphi}$ we have that $H^2\varphi=H(H(\varphi))=-\varphi$, we get that
\begin{align*}
 \phi_S(t) =& \tan^{-1}\frac{HS(t)}{S(t)} \\
 =& \tan^{-1} \frac{-\cos t}{\sin t}\\
 =& \tan^{-1} \frac{\cos t}{-\sin t} \\
 =& \tan^{-1} \frac{\sin (t + \frac{\pi}{2})}{\cos (t + \frac{\pi}{2})}\\
 =& t + \frac{\pi}{2}.
\end{align*}
The instantaneous phase difference for the two basic trigonometric functions is, at it should be expected,
\begin{equation*}
 \phi_S(t) - \phi_C(t)= \left(t + \frac{\pi}{2}\right) - t = \frac{\pi}{2}.
\end{equation*}
This basic test provides an elementary heuristics for the definition of the instantaneous phase of a signal.

\begin{definition}
 Let $f$ be a given signal whose Hilbert transform is given by the function $Hf(t)$. We define the instantaneous phase of $f$ as
 \begin{equation*}
     \phi_f(t)=\tan^{-1}\frac{Hf(t)}{f(t)}.
 \end{equation*}
\end{definition}

Which from our construction of $H$ in terms of the Fourier Transform given in Lemma~\ref{lemma2.1}, can be explicitly given in terms of the signal $f$ by
\begin{equation*}
 \phi_f(t)=\tan^{-1}\frac{(-i\sigma\widehat{f})\,\widecheck{}}{f(t)},
\end{equation*}
where $i$ is the imaginary unit in the complex numbers and $\sigma$ is the sign function. 

The use of these concepts in the analysis of signals in EEG or some other techniques is related to the possibility of the detection of phase differences between signals that are simultaneous and obtained from different regions of the brain. 

\begin{definition}
    Given two signals $f$ and $g$ the absolute instantaneous Phase difference  between $f$ and $g$ is given by
    \begin{equation*}
        \begin{split}
            \Delta_{fg}(t) &= \left| \phi_f(t) - \phi_g(t) \right| \\
            &= \left| \tan^{-1}\frac{Hf(t)}{f(t)} - \tan^{-1}\frac{Hg(t)}{g(t)} \right| \\
            &= \left| \tan^{-1}\frac{(-i\sigma\widehat{f})\,\widecheck{}\,(t)}{f(t)} - \tan^{-1}\frac{(-i\sigma\widehat{g})\,\widecheck{}\,(t)}{g(t)}\right|.
        \end{split}
    \end{equation*}
\end{definition}

Actually, since $\Delta_{fg}$ is itself a function of $t$, we obtain a scalar indicator of phase difference taking mean values of $\Delta_{fg}(t)$. Precisely, for $f$ and $g$ two signals defined in the time interval $I$,
\begin{align*}
    \bar{\Delta}_{fg} &= \frac{1}{|I|}\int_{I} \Delta_{fg}(t)dt
\end{align*}
where $|I|$ is the length of $I$. In fact, when dealing with real signals, as we shall see in the next section, this mean value is computed through the discrete version of the given signal,
\begin{align*}
    \bar{\Delta}_{fg} &= \frac{1}{N}\sum_{i=1}^{N} \Delta_{fg}(t_i) \quad \text{with} \quad N\Delta_{t}=|I|.
\end{align*}

%---------------------------------------------------------------------------------------------
\section{A higher dimensional extension}\label{HigherExtension}
The Hilbert Transform $H$ used in the definition of instantaneous phase of a one dimensional signal in the previous sections is the seed for what, in Harmonic Analysis, is known as the theory of singular integrals. In dimension higher than one, the pioneer works are those of Riesz, Mikhlin and Calder\'{o}n and Zygmund. The most relevant applications of this theory is in the field of regularity of solutions of Partial Differential Equations \cite{Stein1970}. Even when some literature exists, less attention has been paid to the use of these higher dimensional singular integrals to the analysis of multidimensional signals. As in the previous sections, in these sections we shall consider the Riesz singular integrals (Riesz Transforms) in the $n$-dimensional situation with the final purpose of obtaining a definition of phase difference  for images based on two-dimensional Fourier Transforms, or more generally in three-dimensions or even $n$-dimensions. For $f \in L^1(\R^n)$, any integrable function in the space, the Fourier Transform of $f$ is given by 
\begin{equation*}
    \begin{split}
        \widehat{f}(\xi) &= \hat{f}(\xi_1,\xi_2,...,\xi_n)\\
        &= {\int ... \int}_{\mathbb{R}^n} e^{-2\pi i x \cdot \xi}f(x)dx \\
        &= {\int ... \int}_{\mathbb{R}^n} e^{-2\pi i (\xi_1 x_1 +...+\xi_n x_n)}f(x_1,...,x_n)dx_1 ... dx_n.
    \end{split}
\end{equation*}

For $\varphi$ smooth and small ay infinity, the following functions
\begin{equation*}
\psi_j (x)=\left(-i \frac{\xi_j}{|\xi|}\widehat{\varphi}(\xi)\right)^{\smash{\vee}}\quad(x),j=1,...,n
\end{equation*}
with $|\xi|^2=\sum_{j=1}^n \xi_j^2$, $\vee$ the inverse Fourier transform, is well defined, continuous in all of $\R^n$ and tend to zero as $|x|\to\infty.$

 \begin{definition}
     The Riesz Transforms of $\varphi$ are given by
     \begin{equation*}
         R_j\varphi(x)=\left( -i \sigma_j(\xi)\widehat{\varphi}(\xi)\right)\widecheck{}\quad(x)
     \end{equation*}
     where $\sigma_j(\xi)=\dfrac{\xi_j}{|\xi|}$, $j=1,...,n$.
 \end{definition}

The \textit{n}-Riesz transforms of a n-dimensional signal provide a vector of what we could name the pointwise phase of the signal. In fact, we may consider, given $f(x_1,...,x_n)$ the n pointwise phases of $f$ given by
\begin{equation*}
    \Phi^j_f(x_1,...,x_n)=\tan^{-1} \frac{R_jf(x_1,...,x_n)}{f(x_1,...,x_n)}, j=1,...,n.
\end{equation*}
Now we are in position to define and compute a phase difference vector and, if sufficient, a scalar phase difference as the length of this vector. We will refer to this $n$-dimensional phase difference as the \textbf{Divergence Phase Index} (denoted as $DPI$).

\begin{definition}
Given two signals $f$ and $g$ it Divergence Phase Index is given by   
\begin{equation*}
   \overrightarrow{\Delta}_{f,g}(x_1,...,x_n)=(\Phi_f^1(x_1,...,x_n)-\Phi_g^1(x_1,...,x_n),...,\Phi_f^n(x_1,...,x_n)-\Phi_g^n(x_1,...,x_n))
\end{equation*}
and the norm can be considered as absolute phase difference as a nonnegative number associated to the pair $f,g$ of given signals;
\begin{equation*}
    \Delta_{f,g}(x_1,...,x_n)= \| \overrightarrow{\Delta}_{f,g}(x_1,...,x_n) \|.
\end{equation*}
\end{definition}

As in the one-dimensional case, when two \textit{n}-dimensional scalars fields are defined in the same region $\Omega$, we obtain a scalar indicator of the $DPI$ taking the mean values on $\Omega$ of $ \Delta_{f,g}(x_1,...,x_n)$, precisely 
\begin{equation*}
    \bar{\Delta}_{f,g}= \frac{1}{|\Omega|}\ \int_{\Omega} {\Delta}_{f,g}(x_1,...,x_n)dx...dn,
\end{equation*}
where $|\Omega|$ denotes the volume of $\Omega$ in $\R^n$.
Notice finally that the two Riesz transforms are particular instances of the more general singular integrals given for $\theta \in [0,2\pi)$ by $R_{\theta}= \cos \theta R_{1}f + \sin \theta R_{2}f$.
With this family of transforms there is a corresponding family of phases $\Phi_{\theta}f=\tan^{-1}\frac{R_{\theta}f}{f}$ and a corresponding family of phase differences for two given signal $f$ and $g$ $\Delta_{\theta_{f,g}}= ||\Phi_\theta f - \Phi_\theta g ||$.

The following result contains the homogeneity property, of $\bar{\Delta}_{f,g}$ which the one-dimensional DPI shares. Let us point out that this homogeneity property reflects in the fact that $\bar{\Delta}$ provides a measure of the shape differences instead of their intensity. 
\begin{proposition}
    Let \( f \in L^2(\mathbb{R}^n) \) and \(\lambda > 0\) be given. Set \(\widetilde{f}=\lambda f\), then \(\Phi^j_f=\Phi^j_{\widetilde{f}}\). Hence \(\overrightarrow{\Delta}_{f\widetilde{f}}=\overrightarrow{0}\).
\end{proposition}
\begin{proof}
    Since each Riesz transform is linear, we have $R_j{\widetilde{f}}=R_j(\lambda f)=\lambda(R_jf)$, so that $\frac{\displaystyle R_j\widetilde{f}}{\displaystyle \widetilde{f}}=\frac{\displaystyle \lambda(R_jf)}{\displaystyle \lambda f}=\frac{\displaystyle R_jf}{\displaystyle f}$.
\end{proof}

%%-------------------------------------------------------
\subsection{The Riesz transforms and the rotation of the space}
\label{sec:rotation_theory}

To better understand the multidimensional nature of the method and its sensitivity to spatial structure, it is important to review the mathematical foundations underlying the transformations applied to the signal. Now, we focus on the Riesz transforms and how they capture spatial changes while maintaining consistency under rotations. A formal proof is then presented showing how these transforms behave under spatial rotations, highlighting the key relationship between the Riesz transforms and the preservation or modification of directional information in the resulting tensor representations.

Let $\vec{R}=(R_1,\cdots,R_n)$ be the Riesz transform vector operator in $\mathbb{R}^n$ defined for $f\in L^2(\mathbb{R}^n)$ as
\begin{equation*}
    \vec{R}(f)=(R_1(f),\cdots,R_n(f)),
\end{equation*}
where $R_j(f)$ is the $j^{th}-$Riesz transform in $\mathbb{R}$. Let $\rho:\mathbb{R}^n\to\mathbb{R}^n$ be a rotation in $\mathbb{R}^n$. In other words, $\rho$ is a linear orthogonal transformation in $\mathbb{R}^n$ such that $\det\rho=1$ and $\rho^{-1}=\rho^T$. Then
\begin{equation} \label{eq:rot}
    \vec{R}(f\circ \rho)=\rho^T\left[\left(\vec{R}(f)\right)\circ \rho\right].
\end{equation}

\begin{proof}
    Recall first that the Fourier transform commutes with rotations. In fact, changing variables
    \begin{equation*}
        \begin{split}
         \widehat{(f\circ\rho)(\xi)}=&\int_{R^n}f\left(\rho(x)\right)e^{-2\pi ix\xi }dx\\
        =& \int_{R^n}f(y)e^{-2\pi i \rho^{-1}y\xi}dy\\
        =& \int_{R^n} f(y)e^{-2\pi i \rho^Ty\xi}dy\\
        =& \int_{R^n}f(y)e^{-2\pi i y \rho \xi}dy\\
        =& \widehat{f}\left( \rho(\xi) \right)\\
        =& \left( \widehat{f}\circ \rho\right) (\xi).
    \end{split}
\end{equation*}
On the other hand, since
\begin{equation*}
    \widehat{R_jg}(\xi)=-i\frac{\xi_j}{|\xi|}\widehat{g}(\xi),
\end{equation*}
we have that
    \begin{equation*}
        \begin{split}
         \widehat{R_j(f \circ \rho)}(\xi)=&-i\frac{\xi_j}{|\xi|} \widehat{f \circ \rho}( \xi) \\
        =& -i\frac{\xi_j}{|\xi|}\widehat{f}\left(\rho(\xi)\right).
        \end{split}
    \end{equation*}
Now, set $\eta=\rho(\xi)$ so that $\xi=\rho^{-1}(\eta)=\rho^T(\eta)$. Hence $\displaystyle \xi_j=\left(\rho^T(\eta)\right)_j=\sum_{k=1}^{n}\rho_{kj}\eta_{k}$ and $|\xi|=|\eta|$, so that
    \begin{equation*} 
        \begin{split}
            \widehat{R_j(f\circ\rho)}\left(\rho^{-1}(\eta)\right) =& -i\sum_{k=1}^n\rho_{kj}\frac{\eta_k}{|\eta|}\widehat{f}(\eta)\\
            =& \sum_{k=1}^n \rho_{kj}\widehat{R_k(f)}(\eta).
        \end{split}
    \end{equation*}
Changing back $\xi=\rho^{-1}(\eta)$, we have
    \begin{equation*}
        \begin{split}
            \widehat{R_j(f\circ\rho)}(\xi) =& \sum_{k=1}^n\rho_{kj}\widehat{R_k(f)}\left(\rho(\xi)\right)\\
            =& \sum_{k=1}^n \rho_{kj}\widehat{R_k(f)\circ \rho}(\xi).
        \end{split}
    \end{equation*}
Hence, for $j=1,\ldots,n$
    \begin{equation*}
        R_j(f\circ\rho)=\sum_{k=1}^n\rho_{kj}R_k(f)\circ\rho,
    \end{equation*}
    or in matrix form
    \begin{equation*}
        \vec{R}(f\circ\rho)=\rho^T\left[ \left( \vec{R}(f)\right)\circ \rho\right].
    \end{equation*}
\end{proof}

%%=============================================================================================================================
%%======================================== RESULT =============================================================================
\section{Results}\label{sec:results}
%-----------------------------------------------------
\subsection{One-dimensional examples}

As described in the Introduction, phase difference analysis was originally developed for the study of one-dimensional (1D) signals. These methods have been particularly impactful in neuroscience, where they are widely used to investigate synchronization between brain regions using electroencephalography (EEG), intracranial EEG (iEEG), or magnetoencephalography (MEG) recordings.

To demonstrate how our proposed DPI algorithm can be applied to one-dimensional signal analysis, we investigated changes in functional connectivity using iEEG recordings acquired during an epileptic seizure.

We selected a 20-second iEEG recording during the patient's transition from the baseline (interictal) state (first 10 seconds) to the seizure (ictal) state (10 to 20 second)(See Figure~\ref{fig:PLI_SZ}A). The iEEG was recorded across 9 channels at a sampling rate of 200 Hz. For one-dimensional complex signals, the Hilbert transform requires narrowband filtering to correctly extract phase information \cite{Pikovsky_2001}. Because of  this, we bandpass-filtered the data in the $[1, 3]\,\text{Hz}$ range. We computed the DPI values for all pairwise channel combinations in both the interictal and ictal conditions. The comparison of $DPI$ values between these states is shown in Figure~\ref{fig:PLI_SZ}B.
As clearly shown in Figure~\ref{fig:PLI_SZ}B, DPI values during seizure (SZ) are remarkable higher than those measured during the baseline state. These results are expected, since generalized epilepsy in this case exhibits hypersynchronization across brain regions.

\begin{figure}[H]
    \centering
    \includegraphics[width=0.9\textwidth]{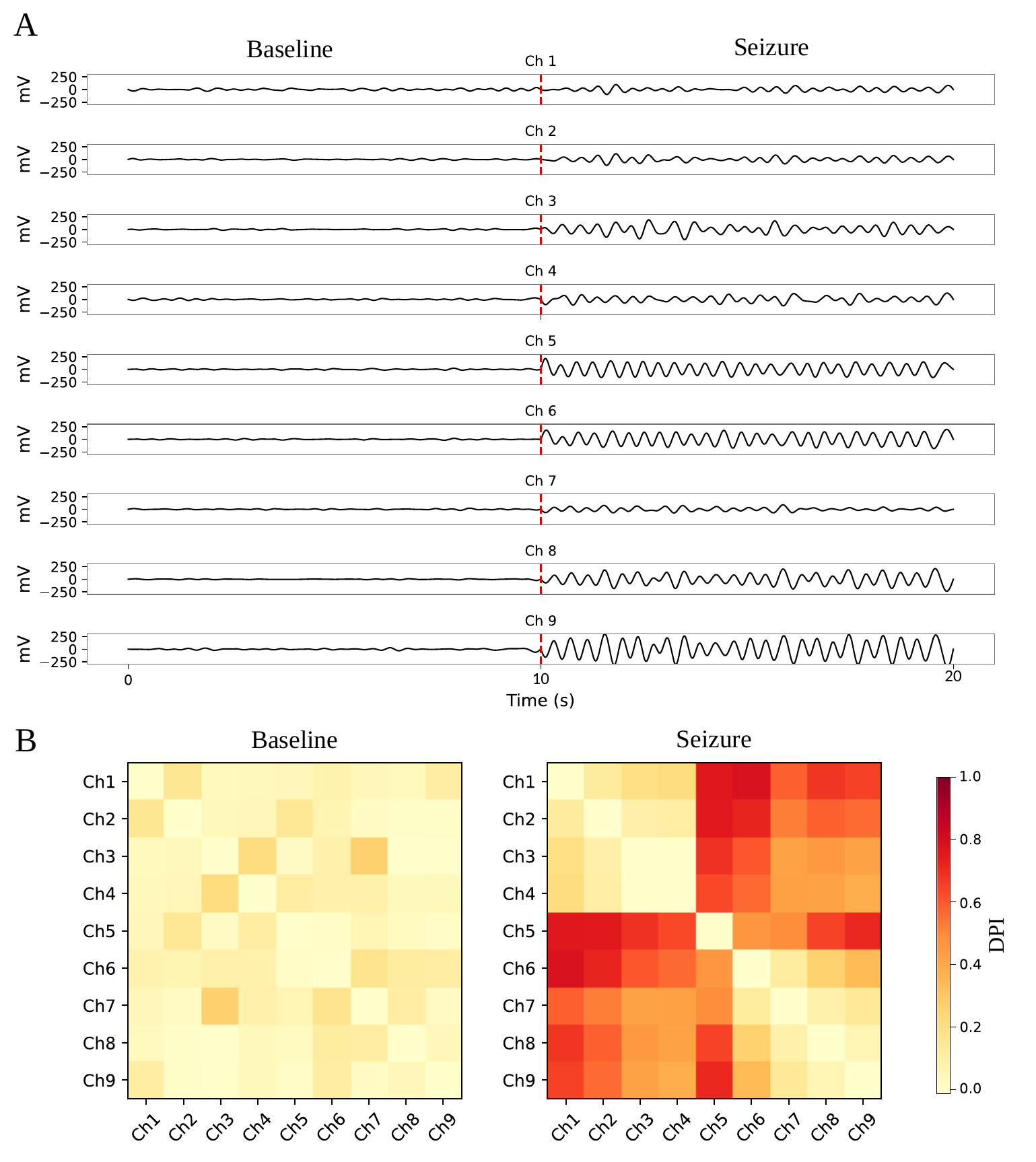}
    \caption{\textbf{(A)} Intracranial electroencephalography (iEEG) signals recorded during baseline (0--10\,s) and seizure (10--20\,s) conditions, filtered in the theta band ($[1, 3]\,\text{Hz}$). \textbf{(B)} Divergence Phase Index ($DPI$) computed across all pairwise channel combinations for baseline (left) and seizure (right) states. }
    \label{fig:PLI_SZ}
\end{figure}
%-----------------------------------------------
\subsection{Two-dimensional examples}

We now extend the application of our proposed $DPI$ method to two-dimensional data, measuring the divergence between pairs of images.
In the first example, we provide an elementary illustration of our method by comparing three simple images: the original image ($O$), a version of the original with reduced intensity (0.5 intensity, denoted $O'$), and a modified image $M$ (with the same intensity as $O$), as shown in Figure~\ref{fig:Fase}A.

We compute the Riesz transform in both the \(x\) and \(y\) directions and obtain the corresponding phase values for all images. Next, we calculate the phase difference vector between image pairs and then compute the norm and average value of this vector. The generalized $DPI$ is calculated for each pair $(O_{ij}, O'_{ij})$, $(O_{ij}, M_{ij})$, and $(O'_{ij}, M_{ij})$, where $O_{ij}$, $O'_{ij}$, and $M_{ij}$ are the restrictions of $O$, $O'$, and $M$, respectively, to square $Q_{ij}$ in a uniform partition of the image into $N_s^2$ subsquares.

To improve visualization of the differences between images, we applied a binarization method to the $DPI$ matrix based on the \textit{Elbow method}~\cite{Thorndike_1953}.

Figure~\ref{fig:Fase}\textit{A second row} shows the resulting two-dimensional $DPI$ matrices.
\[
DPI(O, O') = \left( \bar{\Delta}_{O_{ij}, O'_{ij}} \right)_{i,j=1,\dots,5}, \quad
DPI(O, M) = \left( \bar{\Delta}_{O_{ij}, M_{ij}} \right)_{i,j=1,\dots,5}, \quad
DPI(O', M) = \left( \bar{\Delta}_{O'_{ij}, M_{ij}} \right)_{i,j=1,\dots,5}.
\]
In these visualizations, red indicates no significant difference between corresponding image squares, while green indicates a significant difference. As discussed in Section~\ref{sec:results}, these results confirm that the two-dimensional $DPI$ is insensitive to intensity differences but effectively detects structural or shape differences between images.

Figure~\ref{fig:Fase}B presents $DPI(O, M)$ calculated using different image partition sizes ($N_s^2=4, 5, 6, 8, 9, 12$). This demonstrates that finer partitions enhance the detection of localized modifications in the image, offering improved spatial resolution.

\begin{figure}[H]
    \centering
    \includegraphics[width=0.9\textwidth]{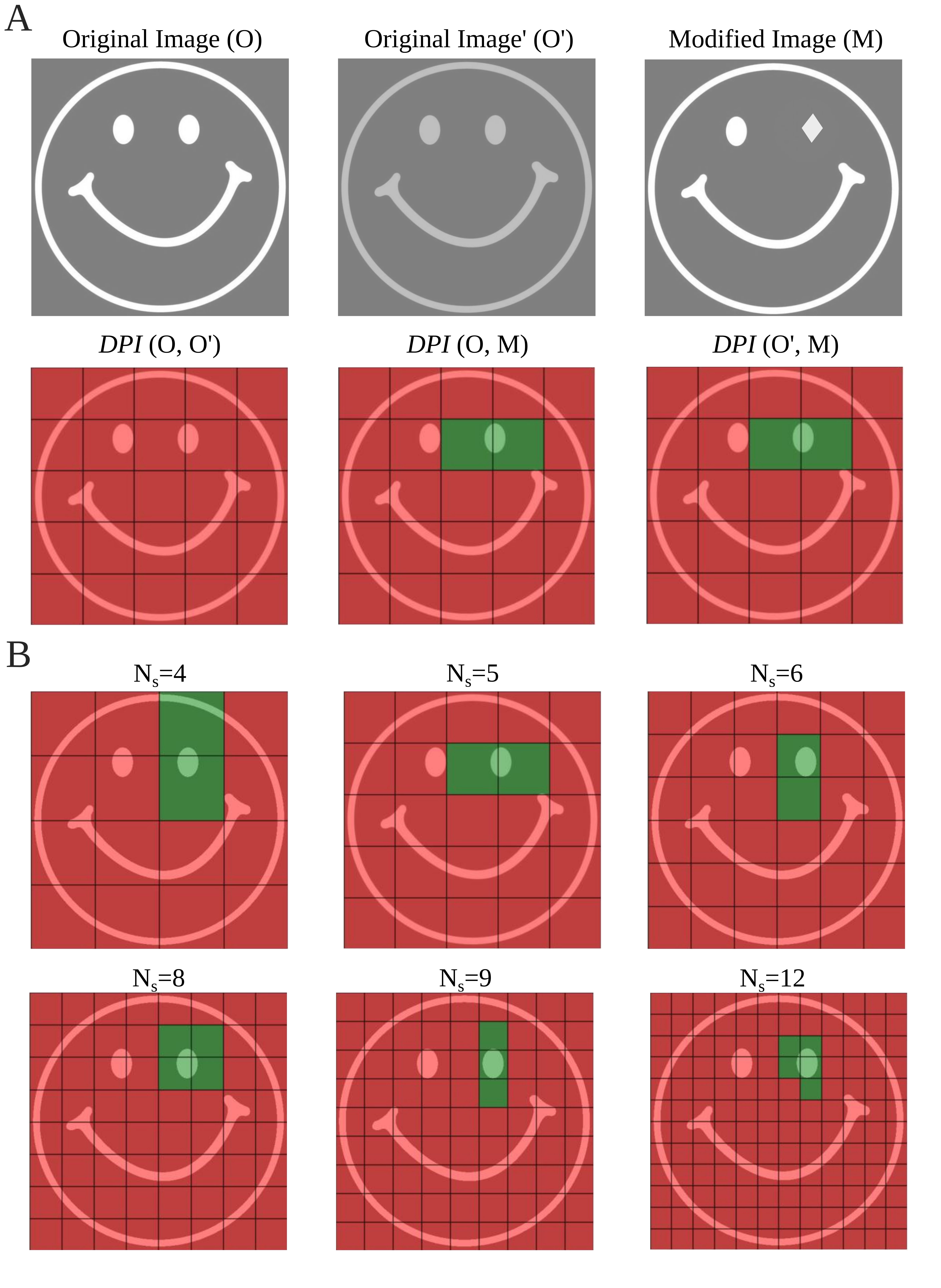}
    \caption{\textbf{(A)} Top row: from left to right, the original image ($O$), the original image with reduced intensity ($O'$), and the modified image ($M$). Bottom row:  Divergence Phase Index ($DPI$) matrices between each pair of images. Red indicates no significant difference between corresponding image squares $(i,j)$, while green indicates a significant difference. \textbf{(B)} $DPI$($O,M$) results for different image partition sizes ($N_s^2 = 4, 5, 6, 8, 9, 12$), showing that finer partitions yield more precise localization of image differences.
}

    \label{fig:Fase}
\end{figure}

To further increase the complexity of the analysis, we examined a real artwork: Vincent van Gogh’s ``Self-Portrait'', painted in July–August 1887 in Paris and currently housed in the van Gogh Museum in Amsterdam\footnote{\url{https://www.wga.hu/html/g/gogh_van/16/selfpo15.html}}. In this case, we intentionally modified one of the eyes in the image, as shown in Figure~\ref{fig:vanGogh}A.

First, we converted the original RGB image to grayscale for both the original version ($O$) and the modified version ($M$). As in the previous example, we also generated a low-intensity version ($O'$) by reducing the intensity of the original to 10\% (Figure ~\ref{fig:vanGogh}B). Figure~\ref{fig:vanGogh}C displays the Divergence Phase Index ($DPI$) computed between all image pairs using a parcellation of $N_s^2 = 17^2=289$ squares, followed by binarization based on the elbow method.

As can be observed, the $DPI $successfully detects the region of the image that was altered, demonstrating robustness to intensity variations and confirming its ability to localize structural changes in realistic, complex images.
 
\begin{figure}[H]
    \centering
    \includegraphics[width=0.9\textwidth]{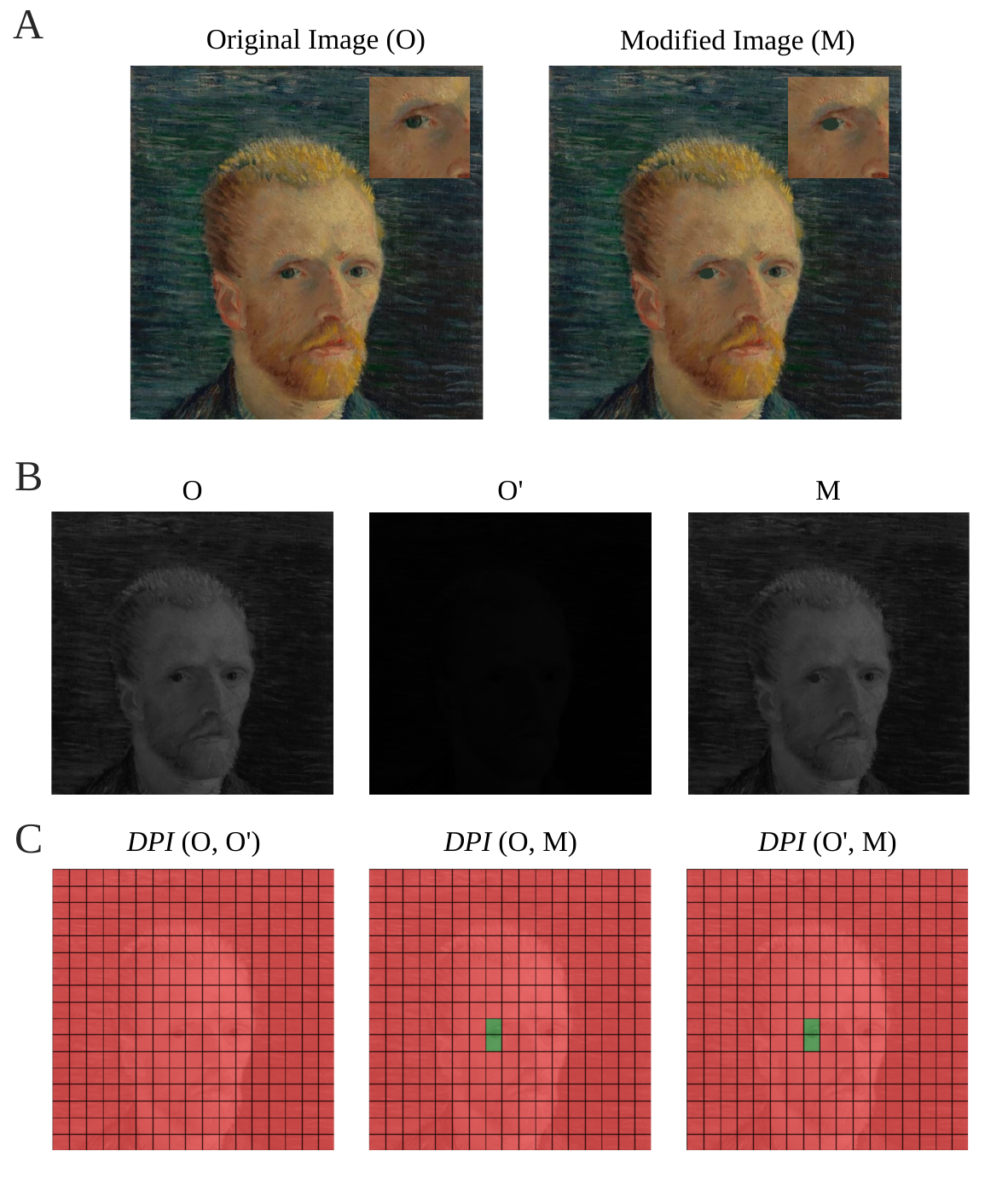}
    \caption{
    \textbf{(A)} Original image (left) and modified image (right) of Vincent van Gogh’s ``Self-Portrait'' with an intentional alteration in the left eye (see zoomed inset). 
    \textbf{(B)} Grayscale conversions of the original image ($O$), the low-intensity version ($O'$; 10\% of original intensity), and the modified image ($M$). 
    \textbf{(C)} Divergence Phase Index ($DPI$) matrices computed for each image pair using a parcellation of $N_s^2 = 17$ squares, followed by binarization based on the elbow method. Red indicates no significant difference between corresponding image squares $(i,j)$, while green indicates a significant difference.
}

    \label{fig:vanGogh}
\end{figure}

%-----------------------------------------------
\subsection{Application of rotation properties in image analysis} \label{sec:app}
Detecting image rotations in real-world contexts—such as photographs or paintings—is relatively easy for both the human eye and many computational algorithms. This is largely due to the anisotropic nature of such images, which often exhibit directional features. For example, van Gogh’s portrait has a predominantly vertical distribution of elements. 

However, when dealing with homogeneous images—such as those acquired through confocal microscopy—rotation detection becomes far less evident. This limitation has even led to the retraction of scientific papers where the same image was used in multiple publications, differing only by a rotation~\cite{retraction2016images}.

In this example, we use a highly isotropic image of the cytoskeleton and membrane systems in neurons from well-preserved biopsy material of patients with Alzheimer’s disease (AD) (Figure~\ref{fig:rot}). The sample was reconstructed from thin sections using conventional electron microscopy and from thick sections using high-voltage electron microscopy. The image was obtained from \texttt{cellimagelibrary.org}.

Based on the property \eqref{eq:rot} explained in Section~\ref{sec:rotation_theory}, to determinate if two matrices/images are the same but rotated we use the following pipeline. To determine whether two images correspond to the same structure under different orientations, the Riesz transform was employed. Prior to comparison, the images were circumscribed within a circle to focus the analysis on the central region and minimize the influence of borders or peripheral content.
In practice, the components $R_1$ and $R_2$ of both the reference image and the target image are first computed. Then, for each possible rotation angle, the components of the reference image are rotated according to the rotation matrix and compared with those of the target image to determine the angle that maximizes the match. In this way, the method does not compare the images directly, but rather evaluates the similarity between their Riesz vector fields, and the angle that produces the best match indicates how much the second image is rotated relative to the first. This approach leverages the sensitivity of the Riesz transform to the orientation of local structures, providing a robust procedure for detecting and quantifying rotations even in images with complex patterns or detailed textures.

The results are presented for five test cases: two nearly imperceptible rotations (1° and 354°), two common rotations (90° and 270°), and two less frequent rotations (137° and 101°).

\begin{figure}[H]
	\centering
	\includegraphics[width=0.7\textwidth]{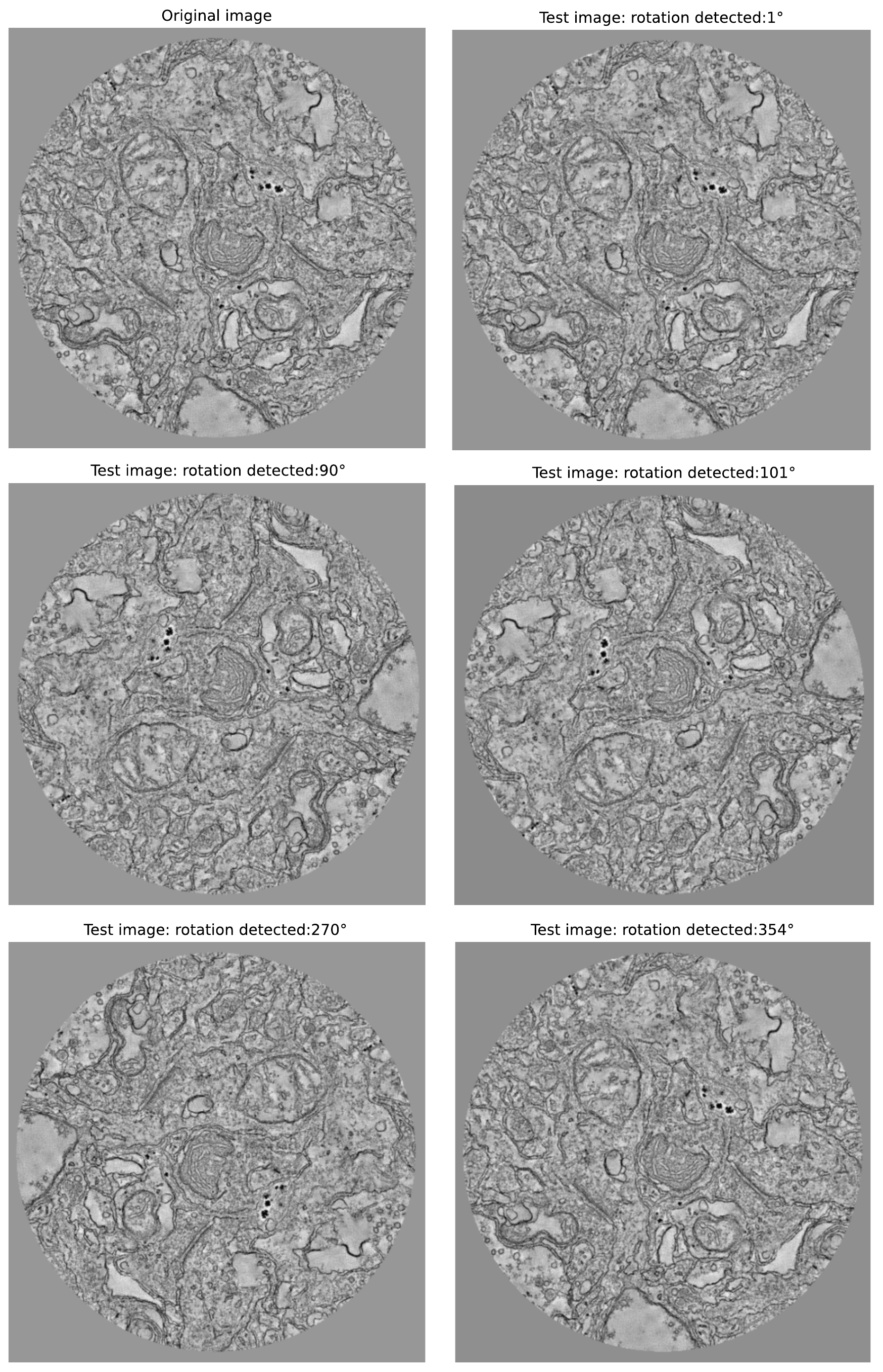}
	 \caption{Visualization of rotation detection using the Riesz transform. The original image (top-left) is compared with several rotated versions (remaining panels) by evaluating the similarity of their $R_1$ and $R_2$ components. The angle that maximizes the match between the vector fields indicates how much each test image is rotated relative to the original. The rotations shown include nearly imperceptible cases, common rotations, and less frequent rotations.}
	\label{fig:rot}
\end{figure}

%==================================================================
%========================= Discussion ==============================

\section{Discussion}
\label{sec:discussion}

In this study, we introduced a novel mathematical framework that generalizes the concept of instantaneous phase difference from one-dimensional signals to higher-dimensional data structures, such as images and spatial maps. We the define a Divergence Phase Index (DPI) as geometry-aware metric grounded in harmonic analysis through the use of Riesz transforms.
\\
Our results shown that DPI can capture structural differences in both synthetic and real-world datasets, including electrophysiological recordings and image-based applications. In the one-dimensional case, DPI robustly detected increased synchronization during epileptic seizures in intracranial EEG (iEEG) recordings. This finding aligns with the idea that generalized epilepsy is characterized by hypersynchronization of the neural networks involved in the pathology~\cite{engel2007epilepsy,Mormann_2005}.
\\
For two-dimensional data, DPI successfully distinguished between images that differ in structure but not necessarily in intensity. This invariance to scalar intensity is particularly important, as many conventional similarity metrics—such as Mean Squared Error, Peak Signal-to-Noise Ratio, Normalized Cross-Correlation, or Histogram Matching~\cite{Wang_2006,gonzalez2009digital}—are highly sensitive to brightness or contrast. By focusing on phase geometry rather than amplitude, DPI is especially well-suited for comparing natural and scientific images, including those from microscopy or digital artwork.
\\
We also evaluated the sensitivity of DPI to geometric transformations, including rotation and reflection. While many natural images exhibit anisotropy that facilitates detection of such transformations, isotropic samples—common in biological microscopy—pose a greater challenge. Our findings show that DPI can detect subtle alterations in spatial structure even in highly isotropic images, offering potential applications in image forensics and data integrity validation. This is particularly relevant given recent cases of scientific retractions due to rotated image reuse.
\\
Theoretical properties such as rotation invariance, homogeneity, and the integration of distributional Fourier transforms further reinforce DPI as a general-purpose analytical tool. 
\\
In summary, the DPI framework provides a versatile, mathematically rigorous, and computationally efficient method for quantifying phase-related differences in both signals and images. Its robustness and flexibility generate new ways for applications in image analysis, biomedical research, and among other fields.

\section{Conclusion}
We presented the Divergence Phase Index (DPI), a geometry-aware metric that extends phase difference analysis from one-dimensional to multidimensional signals and images through the use of Riesz transforms. DPI demonstrated robustness to intensity variations, sensitivity to structural and geometric transformations, and applicability across domains, from detecting hypersynchronization in intracranial EEG to identifying subtle differences in highly isotropic microscopy images. These results highlight DPI’s as potentially  tool  for neuroscience signals analysis, biomedical imaging, and image integrity assessment.

\section*{Funding}
This publication was supported by grand $\#$ MinCyT-FonCyT  PICT-2019 N° 01750 PMO BID; grant CONICET-PUE-IMAL $\#$ 229 201801 00041 CO;  grant CONICET-PIP-2021-2023-GI $\#$11220200101940CO and grant UNL-CAI+D $\#$ 50620190100070LI.

\section*{Author Contributions}

MC: Data curation, investigation, formal analysis, software, visualization, writing - review $\&$ editing. 
HA: Conceptualization, investigation, article supervision, writing: original draft, review $\&$ editing. Funding acquisition.
DM: Article supervision, visualization, writing: original draft, review $\&$ editing.  

\section*{Data Availability Statement}
All data that support the findings of this study are available from the corresponding author upon reasonable request by email.

%============================================================
%	REFERENCE LIST
%-============================================================

%\bibliography{bibliography.bib}
\bibliographystyle{unsrt}

\end{document}